\def\year{2020}\relax
\documentclass[letterpaper]{article} 
\usepackage{aaai20}  
\usepackage{times}  
\usepackage{helvet} 
\usepackage{courier}  
\usepackage{url}  
\usepackage{graphicx} 
\usepackage{multirow}
\usepackage{booktabs}       
\usepackage[export]{adjustbox}
\usepackage{amsfonts,amsthm,amsmath}       
\usepackage{nicefrac}       
\usepackage{microtype}      
\usepackage{color}
\usepackage{bbold} 			

\usepackage{algorithm,algpseudocode}

\newcommand{\noop}[1]{}

\usepackage{filecontents,pgfplots}
\usepgfplotslibrary{colormaps}
\usepackage{pgfplotstable}
\pgfplotsset{
	colormap={test}{[2pt]
    	rgb255=(117,112,179);
        rgb255=(117,112,179);
    },
}
\usepgfplotslibrary{groupplots}
\pgfplotsset{compat=1.14}
\usepackage{color}
\definecolor{cSpec}{RGB}{231,41,138}
\definecolor{cSpecL}{RGB}{117,112,179}
\definecolor{cRSC}{RGB}{102,166,30}
\definecolor{cSC}{RGB}{27,158,219}
\definecolor{cDBSCAN}{RGB}{217,95,2}

\newtheorem{theorem}{Theorem}

\newtheorem*{maincontribution}{Main Contribution}
\newtheorem{auxcontribution}{Auxiliary Contribution}
\theoremstyle{definition}
\newtheorem{definition}{Definition}

\DeclareMathOperator*{\argmin}{arg\,min} 
\DeclareMathOperator*{\argmax}{arg\,max}

\title{Softmax-based Classification is k-means Clustering: Formal Proof, Consequences for Adversarial Attacks, and Improvement through Centroid Based Tailoring}

\author{Sibylle Hess \AND Wouter Duivesteijn \AND Decebal Mocanu}

\begin{document}

\maketitle

\begin{abstract}
We formally prove the connection between k-means clustering and the predictions of neural networks based on the softmax activation layer.  In existing work, this connection has been analyzed empirically, but it has never before been mathematically derived.  The softmax function partitions the transformed input space into cones, each of which encompasses a class.  This is equivalent to putting a number of centroids in this transformed space at equal distance from the origin, and k-means clustering the data points by proximity to these centroids.  Softmax only cares in which cone a data point falls, and not how far from the centroid it is within that cone.  We formally prove that networks with a small Lipschitz modulus (which corresponds to a low susceptibility to adversarial attacks) map data points closer to the cluster centroids, which results in a mapping to a k-means-friendly space.  To leverage this knowledge, we propose Centroid Based Tailoring as an alternative to the softmax function in the last layer of a neural network.  The resulting Gauss network has similar predictive accuracy as traditional networks, but is less susceptible to one-pixel attacks; while the main contribution of this paper is theoretical in nature, the Gauss network contributes empirical auxiliary benefits.
\end{abstract}

\section{Introduction}

The likely cause for the resurgence of neural networks in the deep learning era lies in what people term their ``unreasonable effectiveness'' \cite{2017Sun,2018Fabbri,2018Zhang,2019Flagel}: they work well, so we use them.  This is an understandable reflex. In the enthusiasm with which the scientific community has embraced deep learning, building an understanding of \emph{how} they work is trailing behind.  While papers do exist that provide a deeper understanding of how a classifier arrives at its predictions \cite{2014Henelius,2014Duivesteijn,2016Ribeiro,2018Ribeiro}, such information is typically of a local, anecdotal, or empirical form.  Valuable as that information can be, there is a dearth of fundamental mathematical understanding of what goes on in deep learning.

The last layer of state-of-the-art neural networks computes the final classifications by approximation through the softmax function \cite{1868Boltzmann}.  This function partitions the transformed input space into cones, each of which encompasses a single class.  Conceptually, this is equivalent to putting a number of centroids at equal distance from the origin in this transformed space, and clustering the data points in the dataset by proximity to these centroids through $k$-means.  Several recent papers posed that exploring a relation between softmax and $k$-means can be beneficial \cite{2018Kilinc,2018Peng,schilling2018deep}.  Kilinc and Uysal write in \cite[Section 3.3]{2018Kilinc}:
\begin{quotation}
``One might suggest performing $k$-means clustering on the representation observed in the augmented
softmax layer ($Z$ or \texttt{softmax}$(Z)$) rather than $F$. Properties and respective clustering performances
of these representation spaces are empirically demonstrated in the following sections.''
\end{quotation}
As this quote demonstrates, the current state of scientific knowledge on the relation between $k$-means and softmax is empirical, which leads into our main contribution:
\begin{maincontribution}
We mathematically prove the relation between the transformation performed by softmax, and $k$-means clustering (cf. Theorem \ref{thm:pen=km} and Equation (\ref{eq:kmPred})).
\end{maincontribution}
We show that predictions of neural networks based on the softmax activation function are equivalent to assigning transformed data points to the closest centroid, as known from $k$-means. The centroids are here calculated from the weights of the last layer and the transformation of the data points is given by the output of the penultimate layer. 

\subsection{Auxiliary Contributions}

One can use this relation to explain why softmax-based neural networks are sensitive to one-pixel attacks in image classification \cite{su2019}.  
A picture needs to be classified into one of several conceptual classes, which neural networks nowadays often can do with a high confidence.  Softmax performs this task by evaluating, in the transformed space, in which cone the image falls.  Softmax does not care whether the image is close to the corresponding centroid or far removed: it is merely closest to the centroid in its cone, but the distance may be very large or very small.  A one-pixel attack consists of changing a carefully selected pixel in an input image, in such a way that the neural network assigns the wrong class with high confidence.  Images that are not vulnerable to one-pixel attacks, are likely to lie very close to the corresponding centroid in the transformed space; images that are vulnerable, are likely to lie further away.  The one-pixel attack may force the image across the border of its cone, moving them closer to another centroid.  The softmax cone classification is likely to assign too high a confidence to such images.  What we should do instead, is allow the final layer of a neural network to be less confident in such cases: a neural network should be allowed to have reasonable doubt, just as humans do.

Theoretical work on the effect of misclassification due to small perturbations of the input revolves around the Lipschitz continuity of the function of the network \cite{tsuzuku2018lipschitz}.  Generally, exact computation of the Lipschitz modulus is NP-hard even for two-layer networks, and state-of-the-art methods may significantly overestimate it \cite{2018Virmaux}.  
\begin{auxcontribution}
We theoretically show that in addition to a small Lipschitz modulus, the robustness of a neural net also depends on the proximity with which confidently classified points are mapped to their corresponding centroid (cf.\@ Theorem~\ref{thm:centroidClose}). This establishes a connection between the robustness of a network and its mapping to a $k$-means-friendly space.
\end{auxcontribution}
In reverse, the clustering suitability of the penultimate layer also has an influence on the Lipschitz modulus.
\begin{auxcontribution}
We propose Centroid Based Tailoring as an alternative to the softmax function in the last layer of a neural network, which:
\begin{itemize}
\item is theoretically well-founded through its relation with $k$-means clustering;
\item has competitive performance on records not vulnerable to one-pixel attacks;
\item is able to express reasonable doubt whenever confronted with data points vulnerable to one-pixel attacks.
\end{itemize}
\end{auxcontribution}
\begin{auxcontribution}
We propose an easily integrated proximal minimization update, such that the weights directly return the centroids, according to which the classification is performed. 
\end{auxcontribution}
Our 
experiments show that models which are trained this way consistently achieve a higher test-accuracy. 
The view of network classification as a $k$-means cluster assignment enables the definition of a new activation function based on the Gaussian kernel function. We propose postprocessing of trained neural networks, where the weights of hidden layers are trained such that the output of the penultimate layer is close to the centroids. We empirically show that this procedure reduces the number of images for which a successful one-pixel-attack can be found by a factor ranging from $2.7$ up to $6.8$.

\section{Related Work}

To the best of our knowledge, the literature is devoid of any paper deriving the relation between the softmax output layer and $k$-means clustering, which is what the paper you are currently reading provides.  The closest related papers are \cite{2018Peng,2018Kilinc}, which were discussed in the Introduction; their exploration of the relation between softmax and $k$-means does not go beyond the empirical level.

As softmax activation is practically a linear function, \cite{gal2016} shows in a different context that it extrapolates with unjustified high confidence data points which are very far from the training data. This makes it very sensitive to noise, which partially explains why the neural networks using softmax as a classifier are so easily affected by adversarial attacks.

Closer to our work, \cite{dou2018} proves that in a specific two-layer neural network ---a linear layer followed by a softmax output--- the ratio of the classification probability can be increased using the fast gradient sign method and  Carlini-Wagner’s $L_{2}$ attack. Recently, \cite{yang2018breaking} has formulated language modeling as a matrix factorization problem and identified the fact that the softmax layer with word embedding does not have enough capacity to model natural language. Furthermore, \cite{Sekitoshi2018} demonstrates that this bottleneck appears due to the fact that softmax uses an exponential function for non-linearity and proposes as an alternative a function composed by rectified linear units and sigmoids.

\subsection{Adversarial Attacks}
Further afield, there is a high amount of work on the topic of adversarial attacks in deep neural networks. Most of these publications focus on finding different methods to perform and prevent these attacks. The interested reader is referred to \cite{akhtar2018} and \cite{2019Carlini} for a more detailed survey on these attacks. To the best of our knowledge, paradoxically, not much work exists on the topic of understanding the fundamental properties of neural networks: why are they so sensitive and prone to misclassifying with high confidence input images which are perturbed by a tiny input noise?

The concept of adversarial attacks stems from a 2009 paper on Support Vector Machines \cite{2009Xu}; it has been extended to machine learning in general \cite{2013Biggio} and deep neural networks in particular \cite{szegedy2014}. The latter paper shows that a hardly perceptible perturbation on a testing input image can create the same reproducible classification error across two different models trained on two different subsets of the training data. In \cite{Goodfellow2015}, it is shown that the main problem which creates this error is given by the enforced linearity introduced in deep neural networks to ease the optimization problem, although doubt have been cast on this hypothesis \cite{2016Tanay}.


\section{The Reason Why Softmax is Sensitive to Noise: Relation to $k$-means Clustering}
In this section we prove the relation between the transformation performed by softmax, and $k$-means clustering.  The core concept of the proof lies in Theorem \ref{thm:pen=km} and the subsequent Equation (\ref{eq:kmPred}), in Section \ref{sec:itsclustering}.  Before we can get there, however, we must dedicate a short Section \ref{sec:lipschitz} to Lipschitz continuity.  After the proof and its fallout, we discuss in Section \ref{sec:tailoring} how to adapt the neural network to reduce its overconfidence.

We begin with discussing feedforward networks, mapping points in the $n$-dimensional space to a $c$-dimensional probability vector, where $c$ is the number of classes.  
Let the function of the network have the following form: 
\[F(x)=\sigma\left(f_p(x)^\top W\right).\]
Here, $f_p(x)$ returns the output of the penultimate layer and $\sigma$ is the softmax function. The last layer is linear; its weights are represented by the matrix $W\in\mathbb{R}^{d\times c}$. We assume for ease of notation that the network function has no bias vector. Note that any affine function can be stated as a linear function by increasing the dimension of the input space by one.

\subsection{Lipschitz Continuity and Robustness}\label{sec:lipschitz}
The effect of misclassification due to small perturbations of the input is theoretically framed by the Lipschitz continuity of the function $F$ \cite{tsuzuku2018lipschitz}. A function $f:\mathbb{R}^n\rightarrow \mathbb{R}^c$ is Lipschitz continuous with modulus $L$, if for every $x_1,x_2 \in\mathbb{R}^n$ we have:
$$
    \left\|f(x_1)-f(x_2)\right\| \leq L\left\|x_1-x_2\right\|
$$
Since the Lipschitz modulus of the softmax function is smaller than one, the Lipschitz modulus of the function $F$ is given by $L_p\|W\|$, where $L_p$ is the Lipschitz modulus of the function $f_p$: 
$$
\begin{aligned}
    \left\|F(x_1)-F(x_2)\right\|^2 &\leq \left\|f_p(x_1)^\top W-f_p(x_2)^\top W\right\|\\
    &\leq L_p\left\|W\right\|\left\|x_1-x_2\right\|.
\end{aligned}
$$
If the Lipschitz modulus is small, then points which are close to each other have also close function values. With respect to neural networks, this denotes a robustness measure since the Lipschitz modulus bounds the effect on the classification of small distortions of data points~\cite{szegedy2014,wengExtremeValue,virmaux2018lipschitz}.
\subsection{Softmax and Inherent Clustering Properties}\label{sec:itsclustering}
We collect in the matrix $X\in\mathbb{R}^{n\times m}$ the $m$ training data points, such that point $x_j=X_{\cdot j}$. We notate with $f_p(X)$ the $d\times m$ matrix, having the vector $f_{p}(x_j)$ as column $j$.
\begin{theorem}\label{thm:pen=km}
Let the dimension of the penultimate layer $d$ be at least as large as the number of classes: $d\geq c-1$.
Given a network whose predictions are calculated as $y=\argmax_k f_{p}(x)^\top W_{\cdot k}$, there exist $c$ class centroids $Z_{\cdot k}\in\mathbb{R}^d$, equidistant to the origin, such that every point $x$ is assigned to the class whose center is closest in the transformed space:
\[y=\argmin_k\left\|f_{p}(x)-Z_{\cdot k}\right\|^2.\]
\end{theorem}
\begin{proof}
Since $d\geq c-1$, a vector $v\in\mathbb{R}^d$ exists, such that the vectors $W_{\cdot k}+v$ for $k\in\{1,\cdots,c\}$ have the same norm. The vector $v$ is a solution of the following system of $c-1$ linear equations for $2\leq l\leq m$:
$$
\begin{aligned}
    &\|W_{\cdot 1}+v\|^2 = \|W_{\cdot l}+v\|^2, \quad\text{which is equivalent to:}\\& 2(W_{\cdot 1}-W_{\cdot l})^\top v = \|W_{\cdot 1}\|^2 - \|W_{\cdot l}\|^2,
\end{aligned}
$$
Let $Z_{\cdot k}=W_{\cdot k}+v$. We have:
\begin{align*}
    y &= \argmax_k f_p(x_j)^\top W_{\cdot k} \\
    &= \argmax_k f_p(x_j)^\top W_{\cdot k} +f_p(x_j)^\top v\\
    &= \argmax_k f_p(x_j)^\top Z_{\cdot k}\\
    &= \argmin_k \|f_p(x_j)\|^2- 2f_p(x_j)^\top Z_{\cdot k}+\|Z_{\cdot k}\|^2\\
    &= \argmin_k \|f_p(x_j)- Z_{\cdot k}\|^2.\tag*{\qedhere}
\end{align*}
\end{proof}
Theorem~\ref{thm:pen=km} implies that the one-hot encoded predictions of a neural network are computed as
\begin{align}\label{eq:kmPred}
\hat{Y}=\argmin_Y\|f_p(X)^\top-YZ^\top\|^2 \quad \text{s.t. }Y\in\mathbb{1}^{m\times c}.
\end{align}
The space $\mathbb{1}^{m\times c}$ consists of all binary partition matrices, that are all binary matrices $Y\in\{0,1\}^{m\times c}$ where every row contains exactly one one. In clustering this implies that every point $f_p(x_j)$ is assigned to exactly one cluster (cluster $k$ if $Y_{jk}=1$). For the neural network this models the notion that every point is assigned to exactly one class.

Equation~\eqref{eq:kmPred} is the matrix factorization form of the objective of $k$-means cluster assignments. As a result, class predictions are made according to a Voronoi tesselation of $\mathbb{R}^d$. Often, the activation function of the penultimate layer is the Rectified Linear Unit (ReLU), such that the matrix $f_p(X)$ is nonnegative and the Voronoi tesselation is performed on a nonnegative space. The result of Theorem~\ref{thm:pen=km} and Equation~\eqref{eq:kmPred} does yet hold independently of the employed activation functions, as long as the activation function is Lipschitz continuous. Since the class centers $Z_{\cdot k}$ have equal norms, each Voronoi cell has the shape of a convex cone. 

In general, it makes no difference for the classification accuracy how far the points $f_p(x)$ are to the centroid, as long as they are in the correct cone. The softmax confidence is high, for points which maximize the inner product $f_p(x)^\top Z_{\cdot k}$, where $Z_{\cdot k}$ is the class center of the predicted class, as the following calculation shows:
\begin{align*}
    \sigma(f_p(x)^\top W)_k &= \frac{\exp(f_p(x)^\top W_{\cdot k})}{\sum_{l=1}^c\exp(f_p(x)^\top W_{\cdot l})}\cdot \frac{\exp(f_p(x)^\top v)}{\exp(f_p(x)^\top v)}\\
    &=\frac{\exp(f_p(x)^\top Z_{\cdot k})}{\sum_{l=1}^c\exp(f_p(x)^\top Z_{\cdot l})}.
\end{align*}
As a result, the softmax confidence is high for points $f_p(x)$ which align with the direction of their class center $Z_{\cdot k}$ and have a large norm. However, 
the Lipschitz continuity of $f_p$ also yields the following inequality:
\begin{equation}\label{eq:LipschCentroid}
    \lVert f_p(x)-Z_{\cdot k}\rVert \leq L_p\lVert x-z_k \rVert,
\end{equation}
where $z_k\in\mathbb{R}^n$ such that $f_p(z_k)=Z_{\cdot k}$. We call in what follows the points $z$ which are mapped to a class centroid a prototype. Equation~\eqref{eq:LipschCentroid} demonstrates that the distance to the closest centroid in the penultimate layer is bounded. Points which lie nearby a prototype are mapped proportionally close to their class centroid.
In addition, with regard to robustness, mapping points far a way from the class center is not desirable, as the following theorem shows.
\begin{theorem}\label{thm:centroidClose}
Let $x\in\mathbb{R}^n$ be a data point with predicted class $k$ and let the center matrix $Z$ be computed as in the proof of Theorem~\ref{thm:pen=km}. We assume that $f_p$ is Lipschitz continuous with modulus $L_p$. Any distortion $\Delta\in\mathbb{R}^n$ which changes the prediction of point $\tilde{x}=x+\Delta$ to another class $l\neq k$ has a minimum size of
\begin{align*}
    \lVert \Delta \lVert \geq \frac{\lVert Z_{\cdot l}- Z_{\cdot k}\rVert - \lVert f_p(\tilde{x})-Z_{\cdot l}\rVert - \lVert f_p(x)-Z_{\cdot k}\rVert}{L_p}.
\end{align*}
\end{theorem}
\begin{proof}
Let $x,\Delta$ and $Z$ be as described above. We derive from the triangle inequality and from the Lipschitz continuity the following inequality:
\begin{align}
    \lVert f_p(\tilde{x})-Z_{\cdot k}\rVert &\leq \lVert f_p(\tilde{x})-f_p(x)\rVert +\lVert f_p(x)-Z_{\cdot k}\rVert\nonumber\\
    &\leq L_p\lVert\Delta\rVert + \lVert f_p(x)-Z_{\cdot k}\rVert.\label{eq:triangleLipsch}
\end{align}
The triangle inequality also yields the following relationship:
\begin{equation*}
    \lVert Z_{\cdot l}-Z_{\cdot k}\rVert\leq \lVert f_p(\tilde{x})-Z_{\cdot l}\rVert + \lVert f_p(\tilde{x})-Z_{\cdot k}\rVert.
\end{equation*}
Subtracting $\lVert f_p(\tilde{x})-Z_{\cdot l}\rVert$ yields a lower bound on the distance $\lVert f_p(\tilde{x})-Z_{\cdot k}\rVert$, which we apply in Equation~\eqref{eq:triangleLipsch} to obtain the final bound on the distortion $\Delta$.
\end{proof}
Theorem~\ref{thm:centroidClose} provides three possible explanations for the phenomenon that too often (very) small distortions are sufficient to change the prediction of the neural network: 1) the Lipschitz modulus is large, 2) class centroids are close to each other, or 3) point $x$ or $\tilde{x}$ is not mapped close to their class centroid. The first case addresses the known measurement of robustness via the Lipschitz modulus. The second aspect, the distance between the centroids is maximized with the softmax confidence of the predicted class. Since the centroids are equidistant to the origin, the distance between the centroids is maximial if the centroids are orthogonal. Similarly, the softmax confidence of the predicted class achieves its maximum value $\sigma(f_p(x)^\top W)_k=1$ only if the point $f_p(x)=\alpha Z_{\cdot k}$ points in the same direction as the closest centroid ($\alpha>0$) and if the centroids are orthogonal. Hence, maximizing the softmax confidence of the predicted class goes hand in hand with maximizing the distance of the centroids.

Now, if the Lipschitz modulus is reasonably small and the centroids are far away from each other, then a small distortion resulting in a change of prediction entails that one of the points $x$ or $\tilde{x}$ are not close to its centroid. This motivates the definition of a new confidence measure, reflecting how far a point in the penultimate layer is from its centroid. 

\section{Gauss-Confidence and Centroid Based Optimization}\label{sec:tailoring}
The interpretation of the penultimate layer output as a $k$-means clustering space motivates the consideration of a new confidence function. A natural choice which reflects the proximity to the cluster centroids is the Gaussian kernel function. 
\begin{definition}
Given a function $f_p:\mathbb{R}^n\rightarrow \mathbb{R}^d$ and a centroid matrix $C\in\mathbb{R}^{d\times c}$, we define the \emph{Gauss-confidence} as the vector returned by the function $\kappa(x)$, where for $k\in\{1,\ldots,c\}$
\begin{equation}
    \kappa(x)_k = \exp(-\|f_p(x)- C_{\cdot k}\|^2)\in(0,1].
\end{equation}
\end{definition} 

Defining a confidence measurement based on the Gaussian kernel function is also known from so-called \emph{imposter networks}~\cite{lebedev2018impostor}. Imposter networks learn as many prototypes as there are training examples. That is instead of $c$ centroids, imposter networks consider $m$ centroids which determine the class by means of the average Gauss-confidence of imposters belonging to one class. In contrast to imposter networks and the softmax-classification, we do not employ any normalization. Hence, the proposed Gauss-score does not return a probability vector, which in return enables the reflection of outliers. That is, we say a point $x$ is close to a prototype of the predicted class if it is close to the centroid, resulting in a Gauss-confidence $\kappa(x)_k\approx 1$. An outlier is far away from all cluster centroids, which is reflected in a Gauss-confidence $\kappa(x)_l\approx 0$ for all $l\in\{1,\ldots,c\}$.
\subsection{Gauss-Networks}
Based on the results of Theorem~\ref{thm:centroidClose}, we aim for robust networks which map well-classifiable points close to the corresponding centroid and points which are difficult to classify further away from all centroids. 
To do so, the centroids should lie in the domain of the function $f_p$. Otherwise, all points in class $k$ will have a minimum distance of
\[\delta_k = \min_x\lVert f_p(x)-C_{\cdot k}\rVert\]
to their centroid.
For example, the equidistant centroids represented by the matrix $Z$ in Theorem~\ref{thm:pen=km} possibly attain negative values. Thus, if the function $f_p$ maps to the nonnegative space, which is often the case due to employed ReLu activation functions, penultimate layer outputs might never get close to their centroid.

If we apply the theory of $k$-means clustering, then the optimal centroids are given by the means of all points $f_p(x_j)$ belonging to one class. The $k$-means objective in Equation~\eqref{eq:kmPred} minimizes the whithin-cluster scatter and maximizes the intra-cluster distances. Hence, choosing the centroids according to the $k$-means objective
\[C = f_p(X) Y\left(Y^\top Y\right)^{-1}\]
provides a trade-off between the required proximity of points to their class centroid and the distance between the class centroids.


\begin{algorithm}[t]
\caption{Centroid-based feed forward network}
\begin{algorithmic}[1]
  \Function{TailorNetwork}{$f_p, X,Y$} 
    \State $C\gets f_p(X) Y(Y^\top Y)^{-1}$ \Comment{Centroid update}
    \State $\min_{f_p}=\|f_p(X)^\top-YC^\top\|^2$
  \EndFunction
\end{algorithmic}
\label{alg:kmeansNetwork}
\end{algorithm}

\begin{table*}[!thp]
    \caption{Results on Cifar-10 datasets for the two architectures VGG16~\cite{simonyan2014} and ResNet18~\cite{he2015}, three runs per configuration. For both the traditional network and the Gauss network, we report four columns: test set accuracy, average confidence in the predicted class, percentage of images misled by a one-pixel attack, and the average confidence in those misclassification.  Confidences reported in the left block are softmax-confidences; those in the right block are Gauss-confidences.}
    \label{tbl:attackCifar}
	\centering
	\begin{tabular}{ll|rlrl|rlrl}\toprule
	  &&\multicolumn{4}{c|}{Traditional network}&\multicolumn{4}{c}{Gauss network}\\
	        &&\multicolumn{2}{c}{Prediction}          &\multicolumn{2}{c|}{Attack} & \multicolumn{2}{c}{Prediction}  & \multicolumn{2}{c}{Attack}  \\
	Run &Network & accuracy & conf & rate  & conf & accuracy & conf & rate & conf  \\\midrule
    1 & VGG16        & 92.5\% & 0.97   & 52\% & 0.84 & 92.7\% & 0.70 & 13\%  & 0.59\\
    2 & VGG16        & 93.0\% & 0.97   & 44\% & 0.79 & 93.1\% & 0.78 & 16\%  & 0.70\\
    3 & VGG16        & 92.8\% & 0.97   & 52\% & 0.81 & 93.0\% & 0.81 & 18\%  & 0.68\\
    1 & ResNet18     & 94.4\% & 0.98   & 34\% & 0.77 & 94.2\% & 0.74 &  5\% & 0.61\\
    2 & ResNet18     & 94.8\% & 0.98   & 33\% & 0.76 & 94.8\% & 0.76 &  7\% & 0.49\\
    3 & ResNet18     & 95.0\% & 0.98   & 34\% & 0.79 & 94.8\% & 0.76 &  7\% & 0.63\\
    \end{tabular}
\end{table*}

We roughly outline in Algorithm~\ref{alg:kmeansNetwork} our proposed post-processing step, returning a network which maximizes the Gauss-confidences of correct classifications. Given a network, whose output of the penultimate layer is represented by the mapping $f_p$, a training set $X\in\mathbb{R}^{n\times m}$ and the corresponding class labels $Y\in\mathbb{1}^{m\times c}$, the function \textsc{GaussNetwork} refines the network such that the training points are mapped close to their corresponding centroid. The centroids $C$ are chosen according to the $k$-means optimality criterion and a minimization subject to the network is performed. We do not outline here an implementation of the optimization step, as it is easily integrated into the standard optimization of neural networks by backpropagation. We refer to the networks returned by the function \textsc{GaussNetwork} as Gauss networks.
\section{Experimental Results}
We evaluate the proposed Gauss networks with respect to two of their core properties: robustness and a suitable reflection of classification confidences and prototypes and outliers by the Gauss-confidence. For this purpose we compare popular network architectures with the refined Gauss variant on the Cifar-10~\cite{krizhevsky2009} and MNIST~\cite{lecun1998} datasets. For all experiments we employ a PyTorch implementation based on the network architectures and optimization scheme provided by \url{https://github.com/kuangliu/pytorch-cifar}. We will publish our code upon acceptance of the paper.

We evaluate the robustness of models by means of two attack methods: the one pixel attack~\cite{su2019} and the Fast Gradient Sign Attack (FGSM)~\cite{goodfellow2014explaining}. One of the theoretical strengths of Gauss networks is the possibility to identify outliers. Since the softmax confidence of the predicted class is always larger than $\frac{1}{c}$, we also specify that a point is considered as an outlier if the Gauss confidence of the predicted class is smaller than $\frac{1}{c}$. Since we consider only datasets with ten classes, the outlier threshold is equal to $0.1$.
\subsection{Cifar-10 Results}
We investigate the vulnerability to one pixel attacks on 500 samples of the Cifar-10 test-dataset. One pixel attacks demonstrate a large discrepancy between the perception of the human brain and neural networks. There exists a plethora of example images, in which changing a single pixel leads to the neural network assigning the wrong class with high confidence. Here, we want to investigate whether this effect pertains when a notion of outliers according to the robustness influencing distance to centroids is possible (cf.\@ Theorem~\ref{thm:centroidClose}). 

Table \ref{tbl:attackCifar} lists results for three runs of every network configuration on the Cifar-10 dataset.  We experiment with two network architectures, VGG16~\cite{simonyan2014} and ResNet18~\cite{he2015}.  For every run we report two blocks of numbers; the leftmost block corresponds to statistics regarding the neural network as it is out-of-the-box, and the rightmost block corresponds to the same statistics regarding the corresponding Gauss network.  Hence, like-for-like comparison of numbers in the left and right block is the main competition in these experiments.

Within a block, we report four columns of statistics. The first two columns concern the prediction of the network on the test set: accuracy and average confidence are reported.  The last two columns concern the vulnerability of the network to one-pixel attacks: we report the percentage of images for which the most harmful one-pixel attack results in misclassification, and the average confidence of the network for these misclassifications.  All confidence reported in the last column per block of Table \ref{tbl:attackCifar} is misplaced; this number should be as close to zero as possible.

Notice that due to the network construction, confidences in the left and right block of the table are not the same: the left block contains confidence values as reported through softmax, and in the right block we report average Gauss-confidences.

When comparing the leftmost columns in each block, we see that the Gauss network essentially maintains the predictive accuracy of the original networks. Sometimes the accuracy is a little bit better, sometimes it is a little bit worse, but overall there seems to be little difference.  Comparison of the second columns shows that confidence of the Gauss network is markedly lower than that of the traditional network.  This is hardly surprising: we set out to reduce the network's vulnerability to one-pixel attacks by building in reasonable doubt, so one would expect the Gauss network to doubt itself more on all cases.  Since the reduced confidence does not come with a drop-off in predictive accuracy, we claim that this is a desirable result.  The third columns of both blocks reveal that we have accomplished our mission: the percentage of images vulnerable to one-pixel attacks has dropped dramatically (by a factor between roughly $2.7$ and $6.8$).  Comparison of the final columns show that the Gauss network has a substantially lower misplaced confidence in its misclassifications than the traditional network has, although this might be an artefact of the overall lower confidence.
\begin{figure}[t!]
\centering
\pgfplotsset{
GaussSty/.style={scatter, only marks , fill opacity=0.5, fill=cSpecL,
     scatter/use mapped color={
                draw=cSpecL,
                fill=cSpecL,
            },scatter src=explicit,
     visualization depends on ={\thisrow{z} \as \perpointmarksize},
     scatter/@pre marker code/.append style={/tikz/mark size=sqrt(\perpointmarksize), /tikz/mark color=cSpecL}},
tradSty/.style={scatter, only marks,fill opacity=0.5, fill=cRSC, 
     scatter/use mapped color={
                draw=cRSC,
                fill=cRSC,
            },scatter src=explicit,
     visualization depends on ={\thisrow{z} \as \perpointmarksize},
     scatter/@pre marker code/.append style={/tikz/mark size=sqrt(\perpointmarksize)}},
}

\begin{tikzpicture}
\begin{groupplot}[group style={group size= 1 by 2, horizontal sep=0.5cm, vertical sep=1.5cm},
    	height=.34\textwidth,
    	width=.5\textwidth,
    	legend columns =-1,
        title style={yshift=-1.5ex}, 
        axis lines = left, xmin=-0.02, xmax=1.1]
 	\nextgroupplot[title={VGG16},
 	ymin=0,ymax=1.1,ylabel={confidence $\tilde{x}$} ,
    legend to name=zelda]
	    \addplot[tradSty] table[meta=z]{attackVGG.dat};
	    \addlegendentry{softmax};
	    \addplot[GaussSty] table[meta=z]{attackGaussVGG.dat};
	    \addlegendentry{Gauss};
	\nextgroupplot[title={ResNet},
 	ymin=0,ymax=1.1,ylabel={confidence $\tilde{x}$} , xlabel={confidence $x$}]
	    \addplot[tradSty] table[meta=z]{attackResNet.dat};
	    \addplot[GaussSty] table[meta=z]{attackGaussResNet.dat};
\end{groupplot}
\end{tikzpicture}\\

\pgfplotslegendfromname{zelda}
\caption{Comparison of the confidence of the predicted class for points $x$ in the test set and the corresponding one pixel attack $\tilde{x}$ for the two architectures VGG16 and ResNet. The size of a point is equal to the square root of the number of points which have the corresponding confidence. Best viewed in color.}
\label{fig:statConf}
\end{figure}
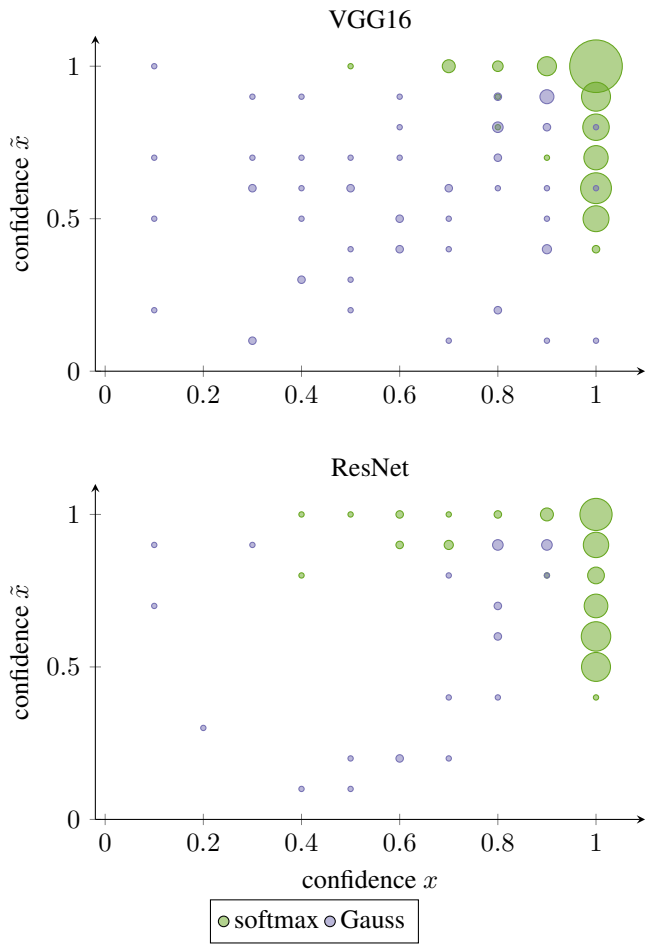

This poses the question if there does also exist a low threshold for the softmax confidence, which decreases the attack rate if we accept only successful attacks with a confidence higher than the given threshold. 
We inspect the distribution of softmax and Gauss confidences of successfully performed one pixel attacks in Figure~\ref{fig:statConf}. We plot the confidence of test data point $x$ against the confidence of its one pixel alteration $\tilde{x}$, which results in a change of prediction. The size of each point corresponds to the square root of the number of occurrences of the denoted confidence pair. We observe that the Gauss-confidences are almost uniformly distributed over the whole space while the softmax confidence of the one pixel attacks $\tilde{x}$ concentrates in the interval $[0.5,1]$. Hence, the attack rate of softmax networks would only drop if we set a threshold higher than $0.5$. The low attack rates of the Gauss confidence are therewith not an artefact of the averagely lower Gauss confidence.

\begin{figure}[!t]
\centering
\pgfplotsset{
nmStyle/.style={cSpec,mark options ={cSpec},mark repeat={1}, thick, error bars/.cd,y dir = both, y explicit},
ncStyle/.style={cSpecL,mark options ={cSpecL},mark repeat={1}, thick, error bars/.cd,y dir = both, y explicit},
nbStyle/.style={cRSC,mark options ={cRSC},mark repeat={1}, thick, error bars/.cd,y dir = both, y explicit},
nfStyle/.style={cSC,mark options ={cSC},mark repeat={1}, thick, error bars/.cd,y dir = both, y explicit},
}

\begin{tikzpicture}
    \begin{groupplot}[group style={group size= 1 by 1},
    	height=\columnwidth,
    	width=\columnwidth,
        legend columns =2, title style={yshift=-1.5ex},
        ymax=1.1,axis lines = left]
        \nextgroupplot[ylabel={Accuracy/Confidence}, xlabel={$\epsilon$},
        	title={MNIST results},
        	legend to name=zelda]
        	\addplot+[nmStyle]  table[x=x,y=y,y error =z] {acc.dat};
            \addlegendentry{Accuracy Traditional};
        	\addplot+[ncStyle]  table[x=x,y=y,y error =z] {acc_mf.dat};
            \addlegendentry{Accuracy Gauss network};
            \addplot+[nbStyle]  table[x=x,y=y,y error =z] {conf_mf.dat};
            \addlegendentry{Gauss confidence};
            \addplot+[nfStyle]  table[x=x,y=y,y error =z] {conf.dat};
            \addlegendentry{Softmax confidence};
    \end{groupplot} 
\end{tikzpicture}\\

\pgfplotslegendfromname{zelda}
\caption{Accuracy and confidence of the FGSM attack on the MNIST data for the traditional and Gauss-networks. The parameter $\epsilon$ denotes the step-size. Best viewed in color.}
\label{fig:noisePlot}
\end{figure}
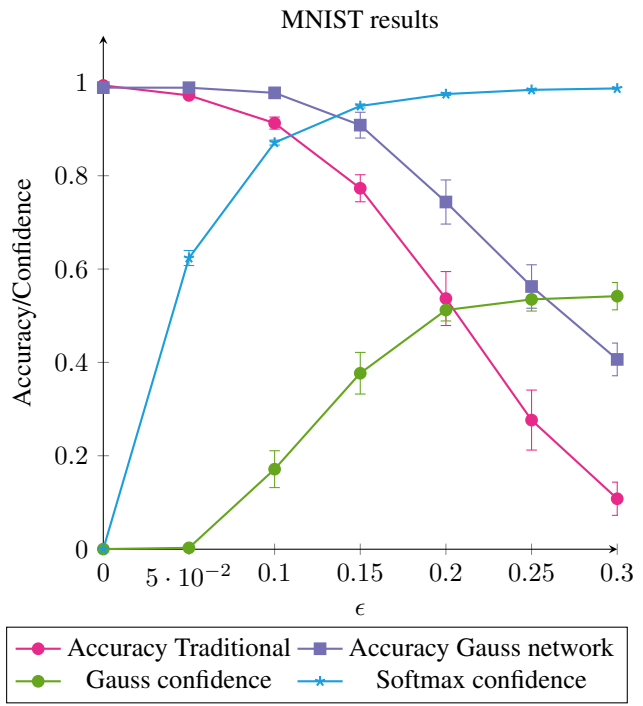
\subsection{MNIST results}
\begin{figure}[t]
  \centering
  \begin{tabular}{rccccc}
  \multirow{4}{*}{\rotatebox{90}{prototypes$\qquad$}}
    & 0.92 & 0.93 & 0.76 & 0.88 & 0.90 \\
    &  \includegraphics[width=0.13\columnwidth,valign=m]{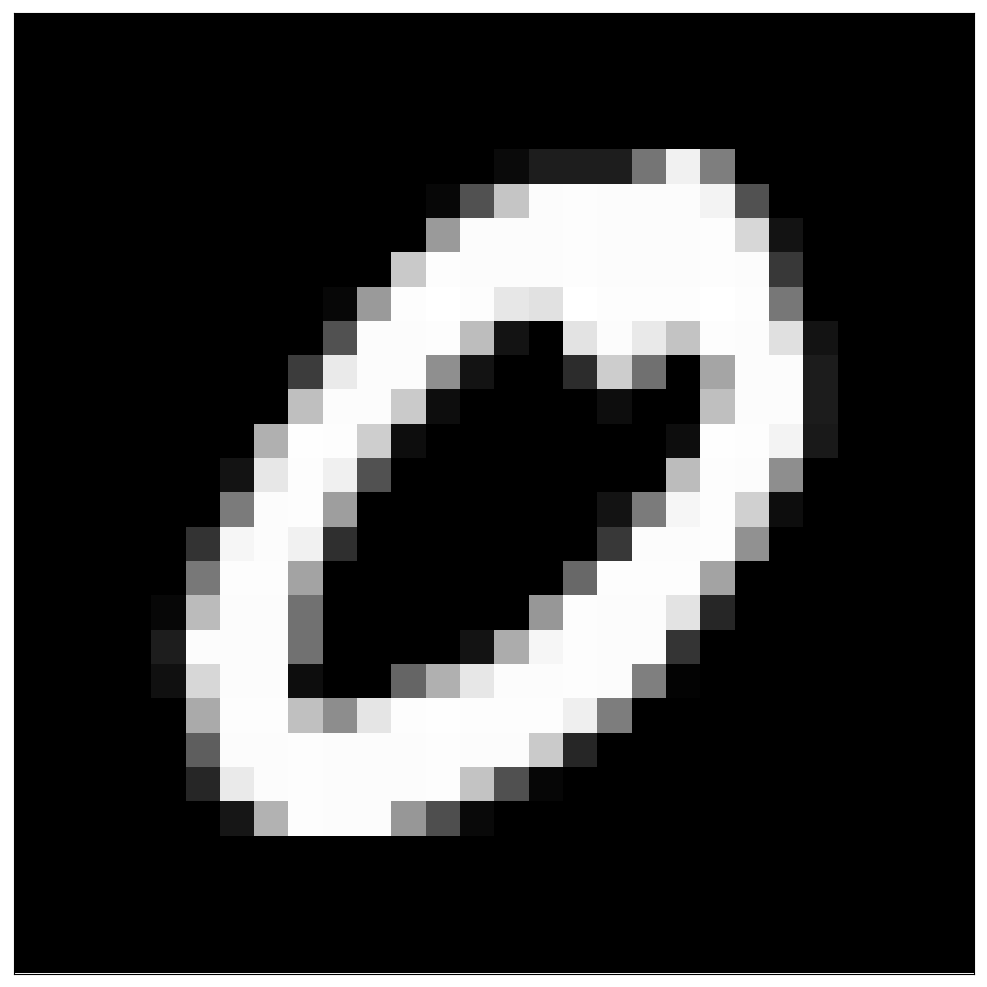}
    &  \includegraphics[width=0.13\columnwidth,valign=m]{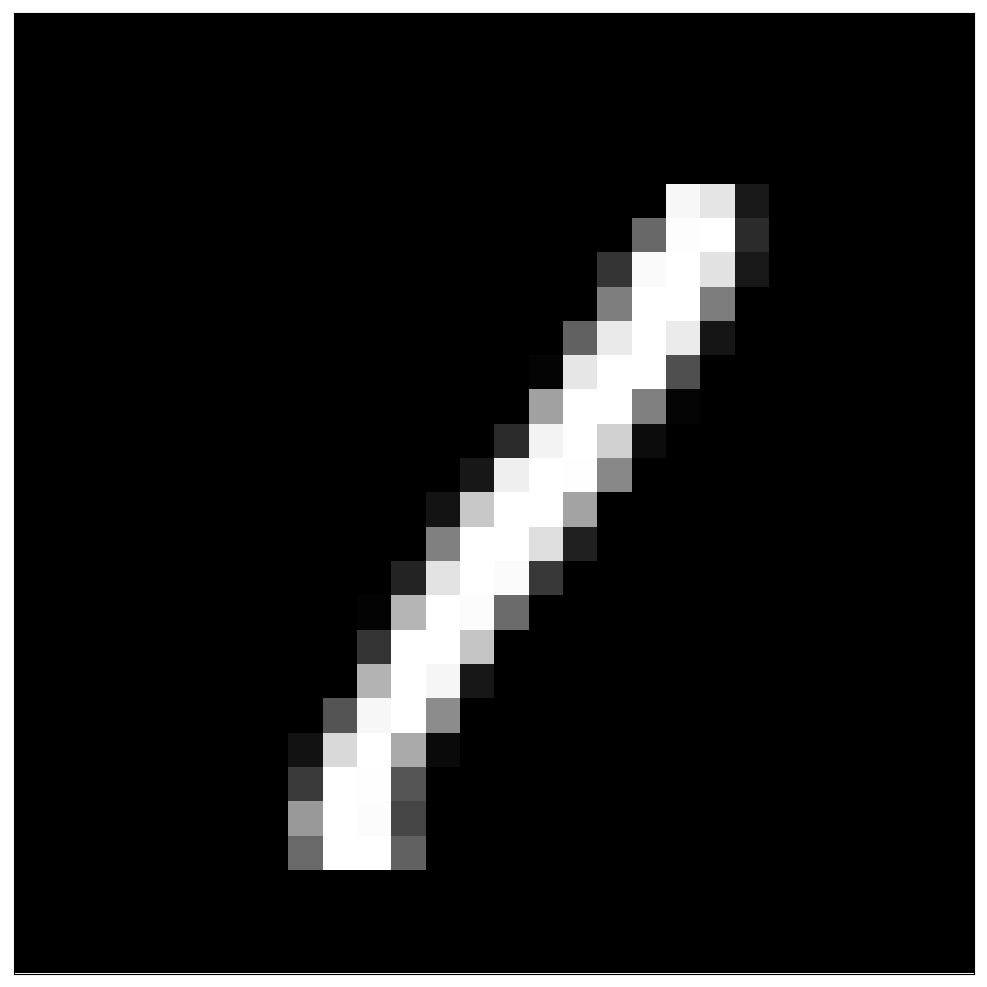}
    &  \includegraphics[width=0.13\columnwidth,valign=m]{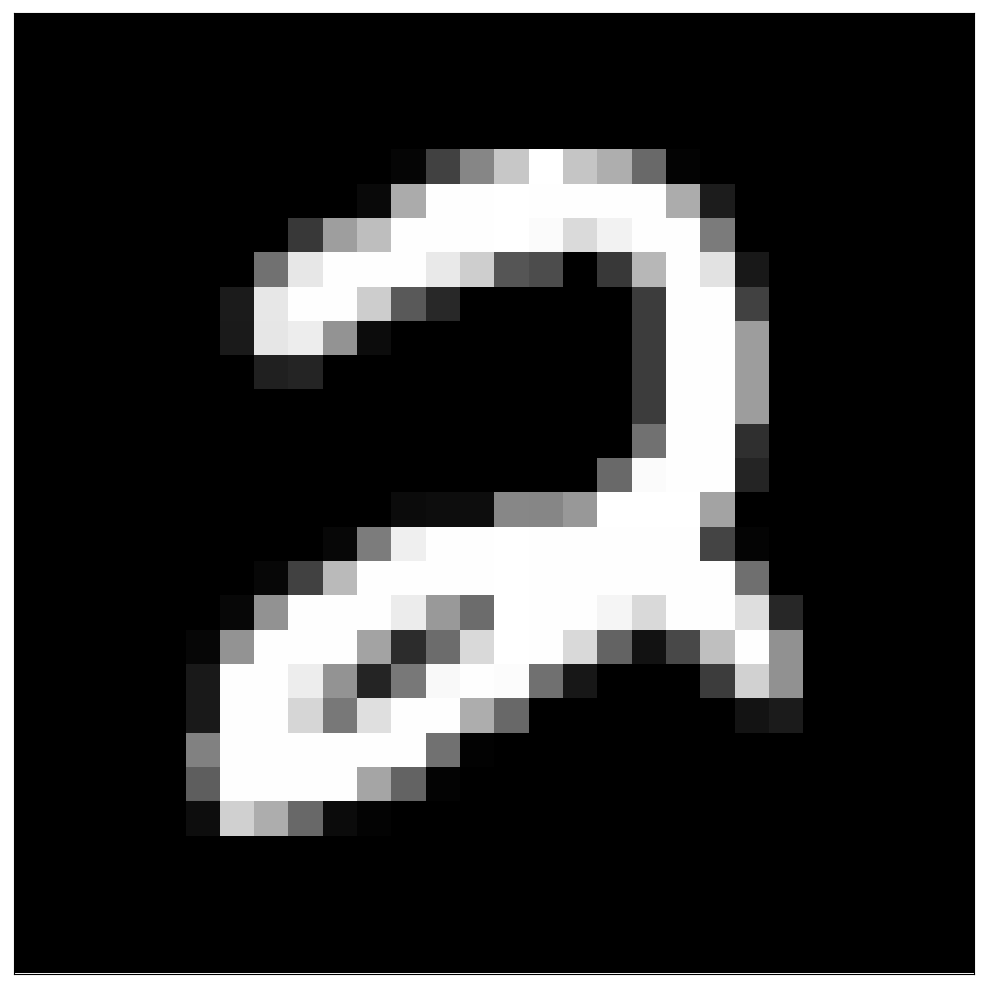}
    & \includegraphics[width=0.13\columnwidth,valign=m]{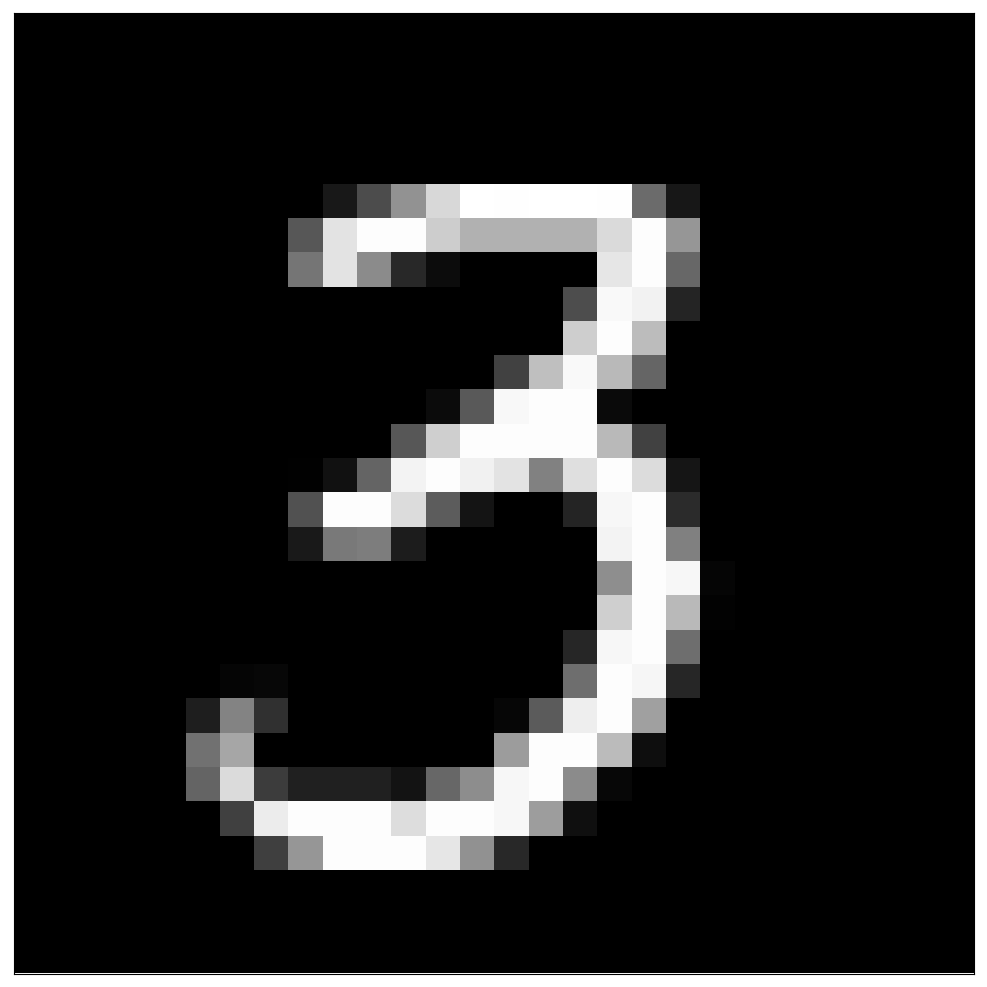}
    & \includegraphics[width=0.13\columnwidth,valign=m]{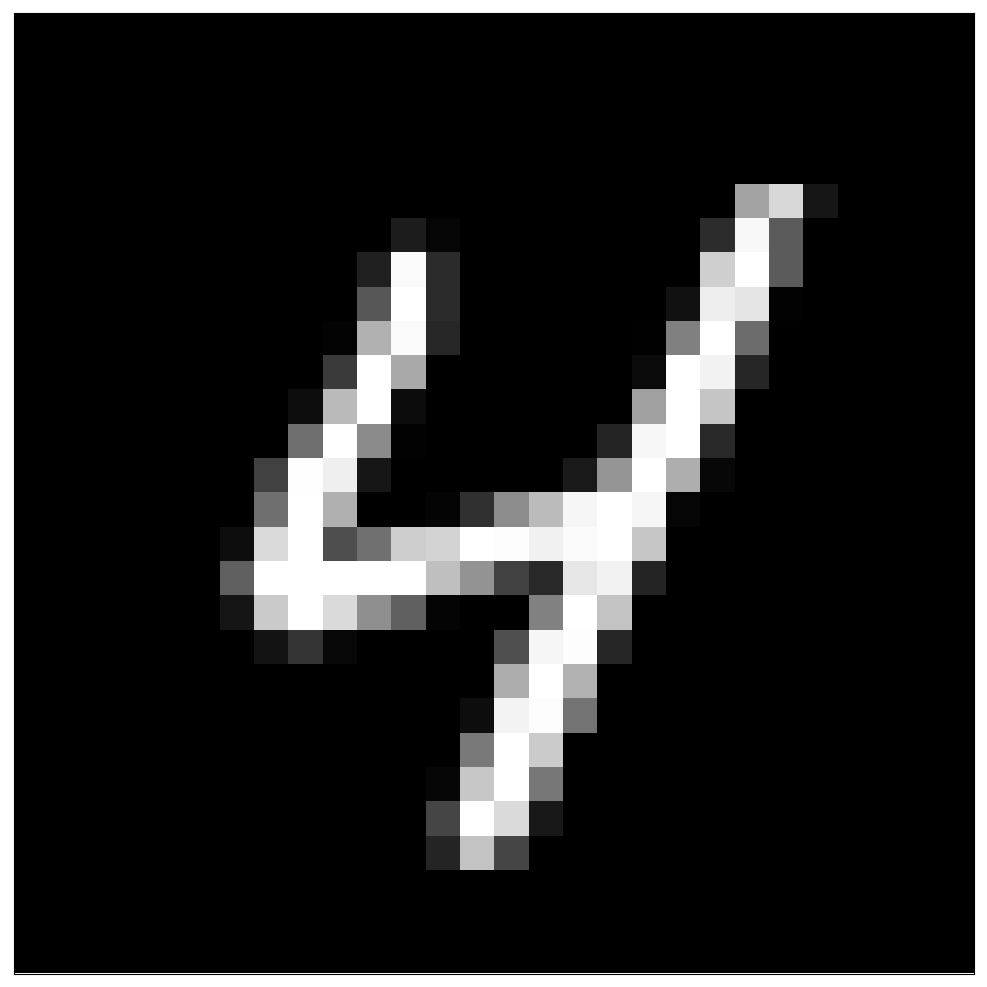} 
    \\
    & 0.62 & 0.82 & 0.84 & 0.60 & 0.73\\
    &  \includegraphics[width=0.13\columnwidth,valign=m]{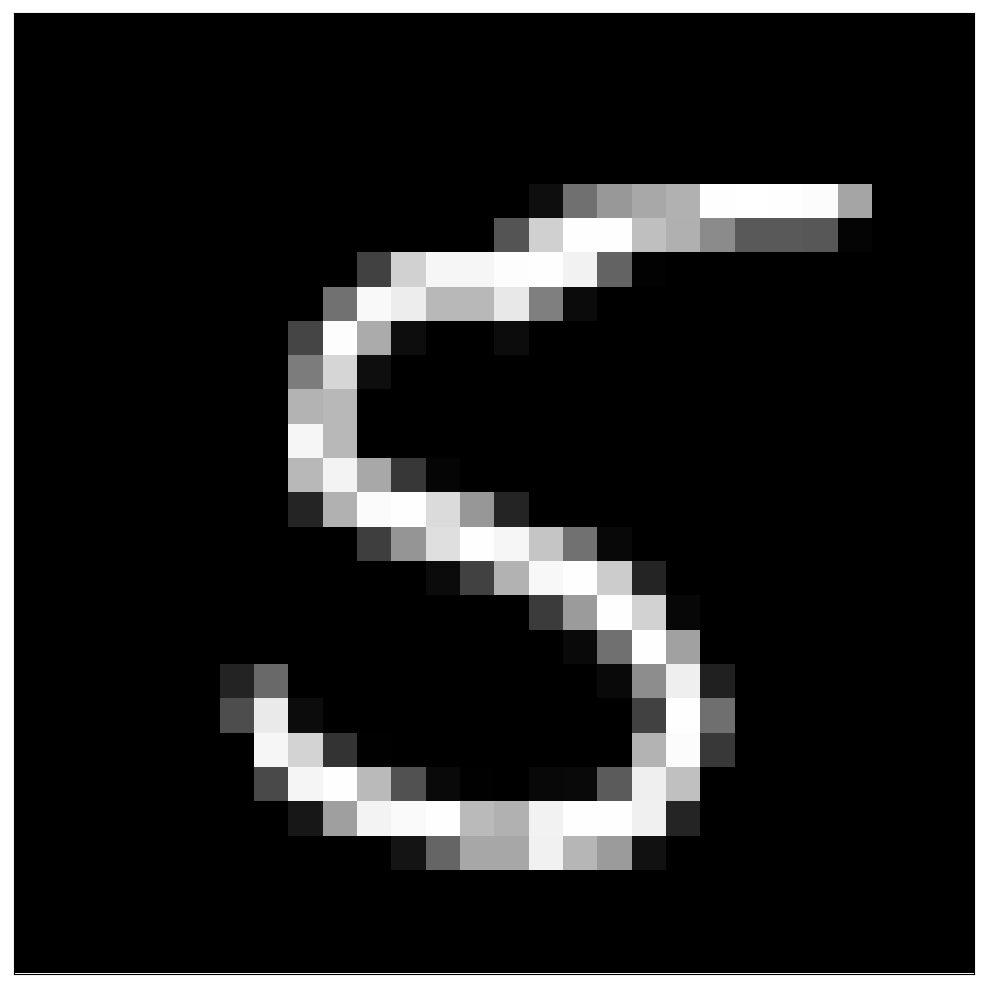} 
    &  \includegraphics[width=0.13\columnwidth,valign=m]{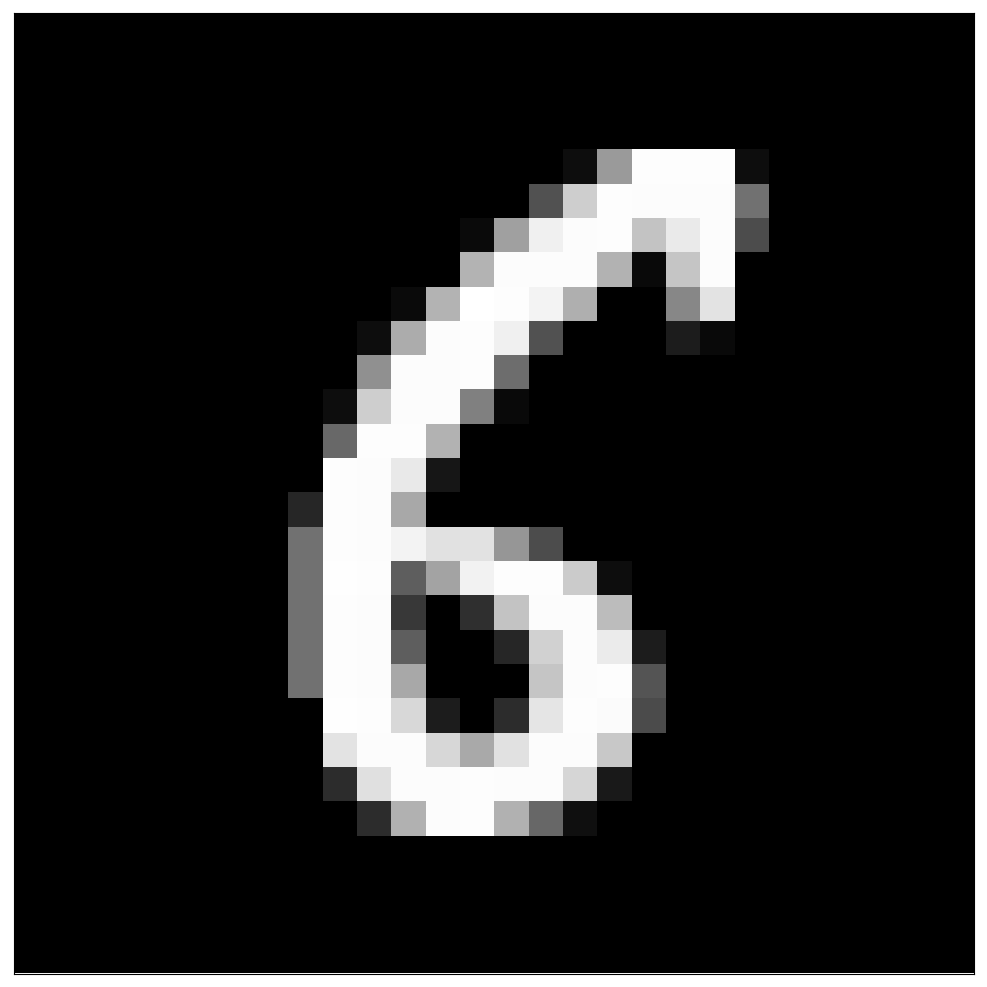}
    &  \includegraphics[width=0.13\columnwidth,valign=m]{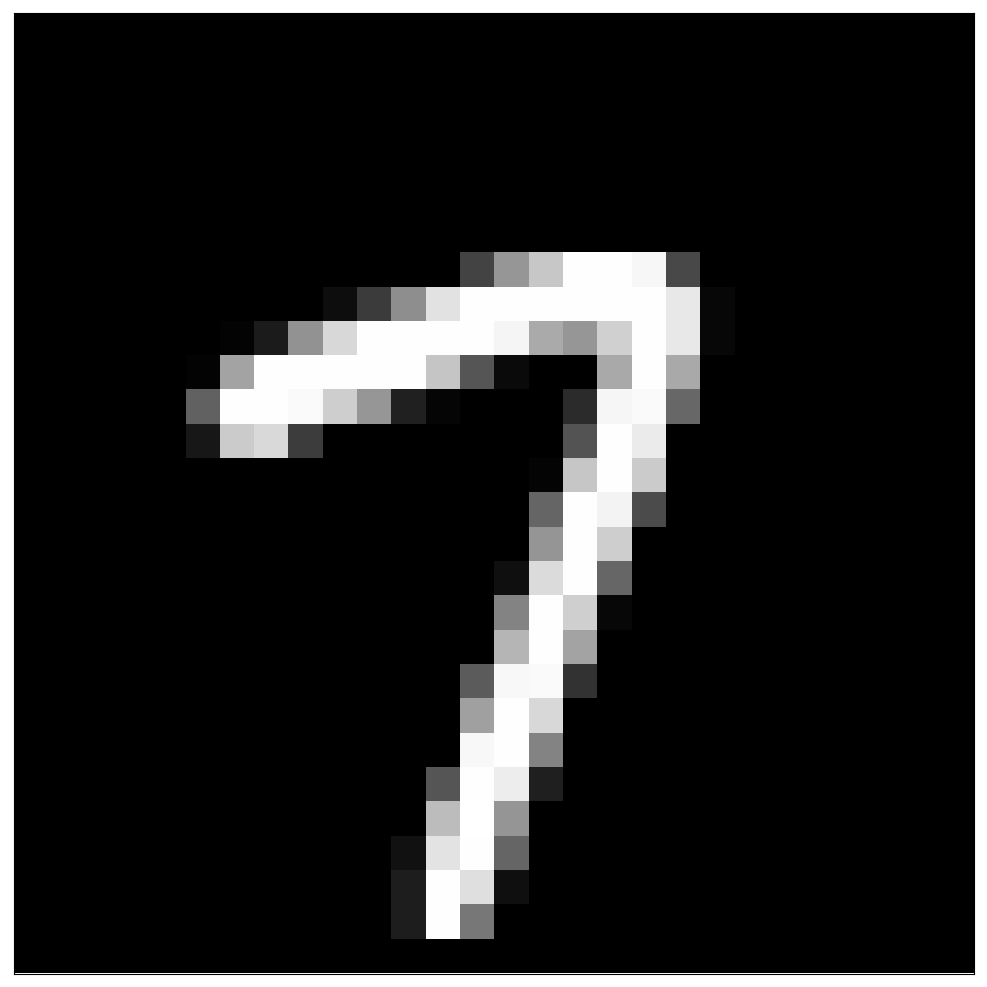}
    & \includegraphics[width=0.13\columnwidth,valign=m]{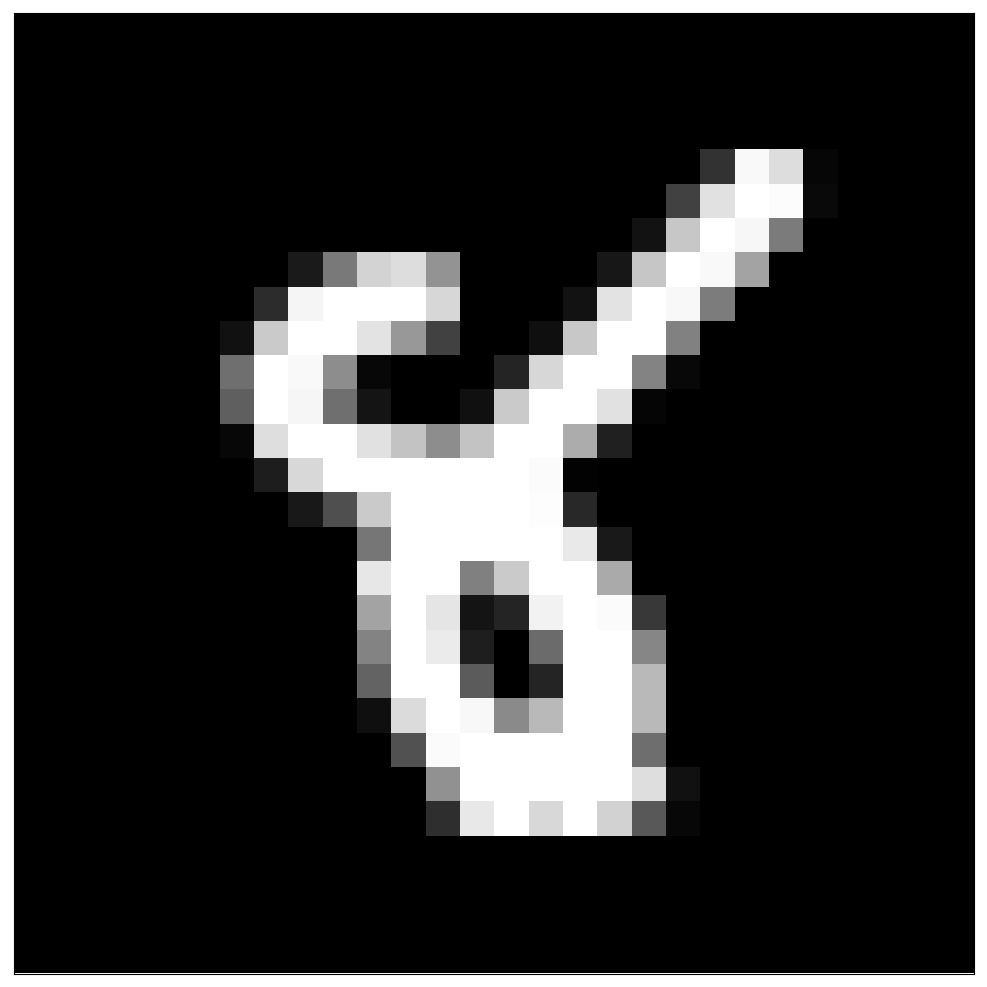}
    & \includegraphics[width=0.13\columnwidth,valign=m]{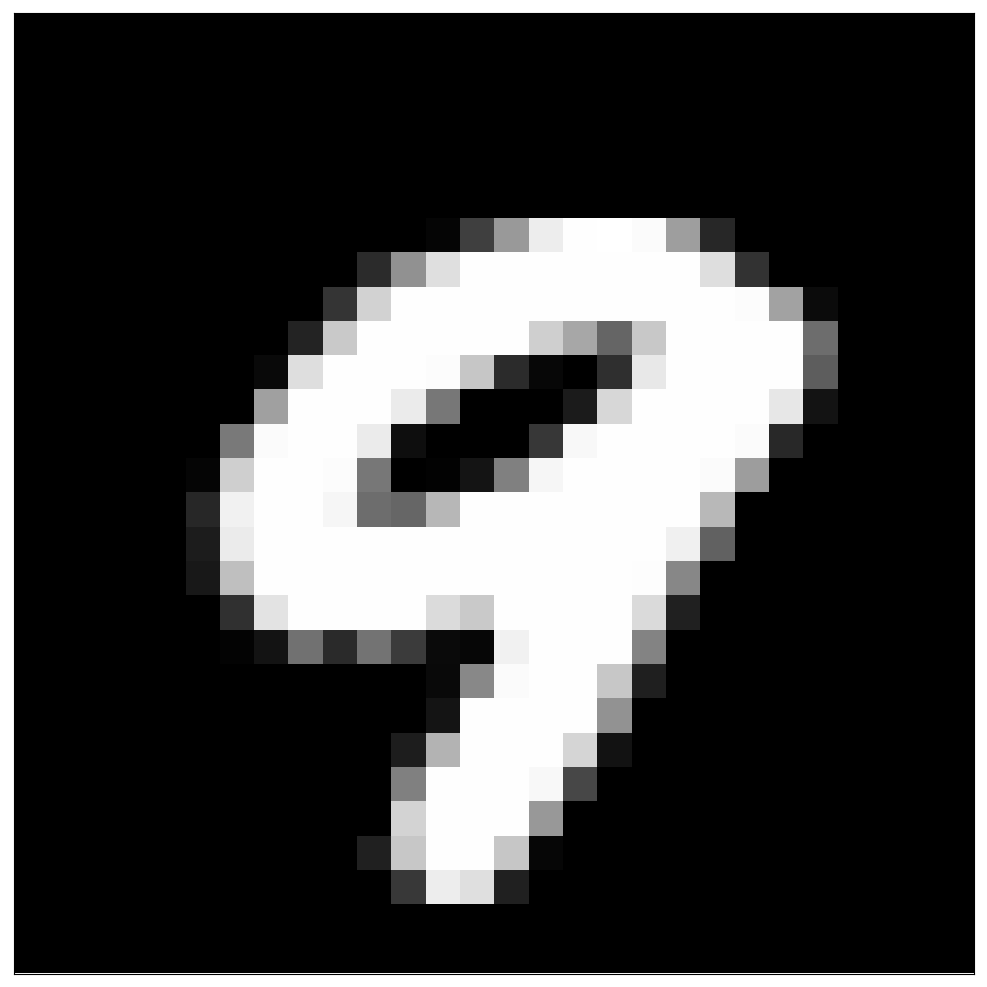} 
    \\[2em]
    \multirow{4}{*}{\rotatebox{90}{outliers$\qquad$} }   
    & $10^{-28}$ & $10^{-39}$ & $10^{-38}$ &  $10^{-27}$ & $10^{-32}$\\
    &  \includegraphics[width=0.13\columnwidth,valign=m]{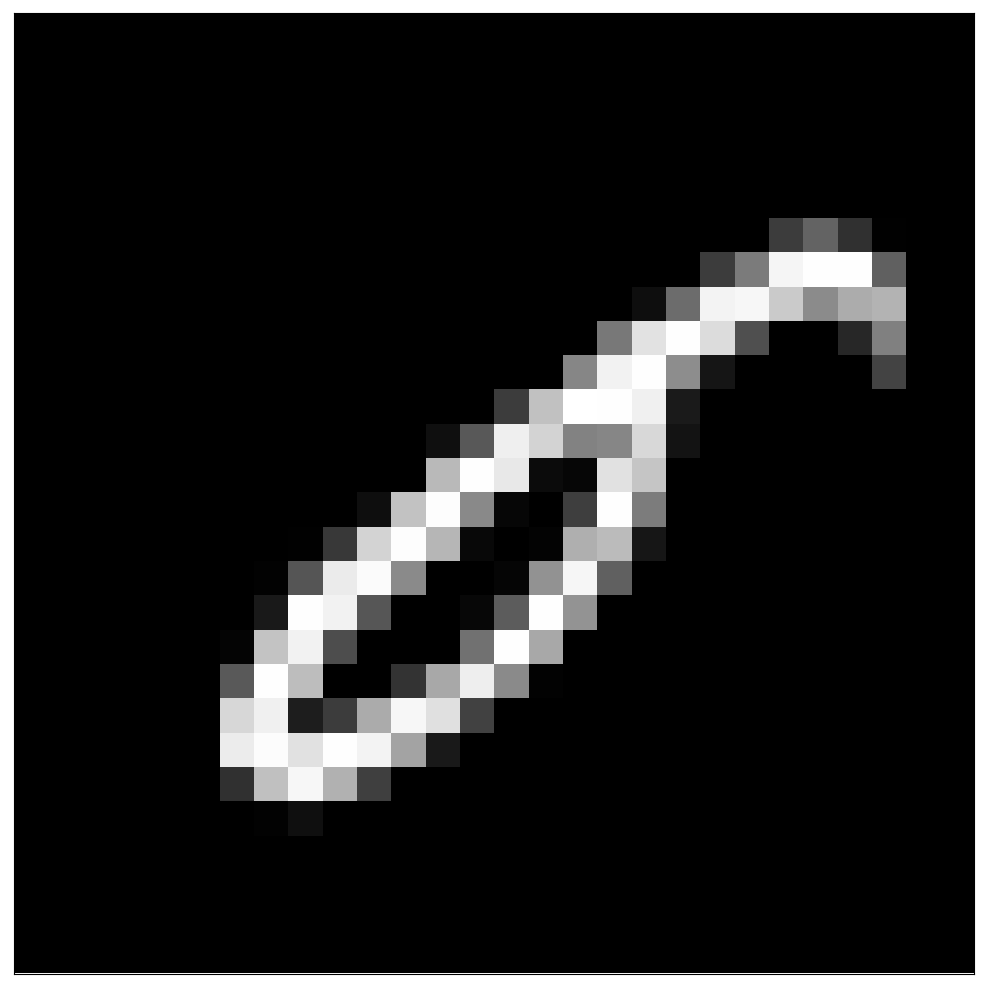} 
    &  \includegraphics[width=0.13\columnwidth,valign=m]{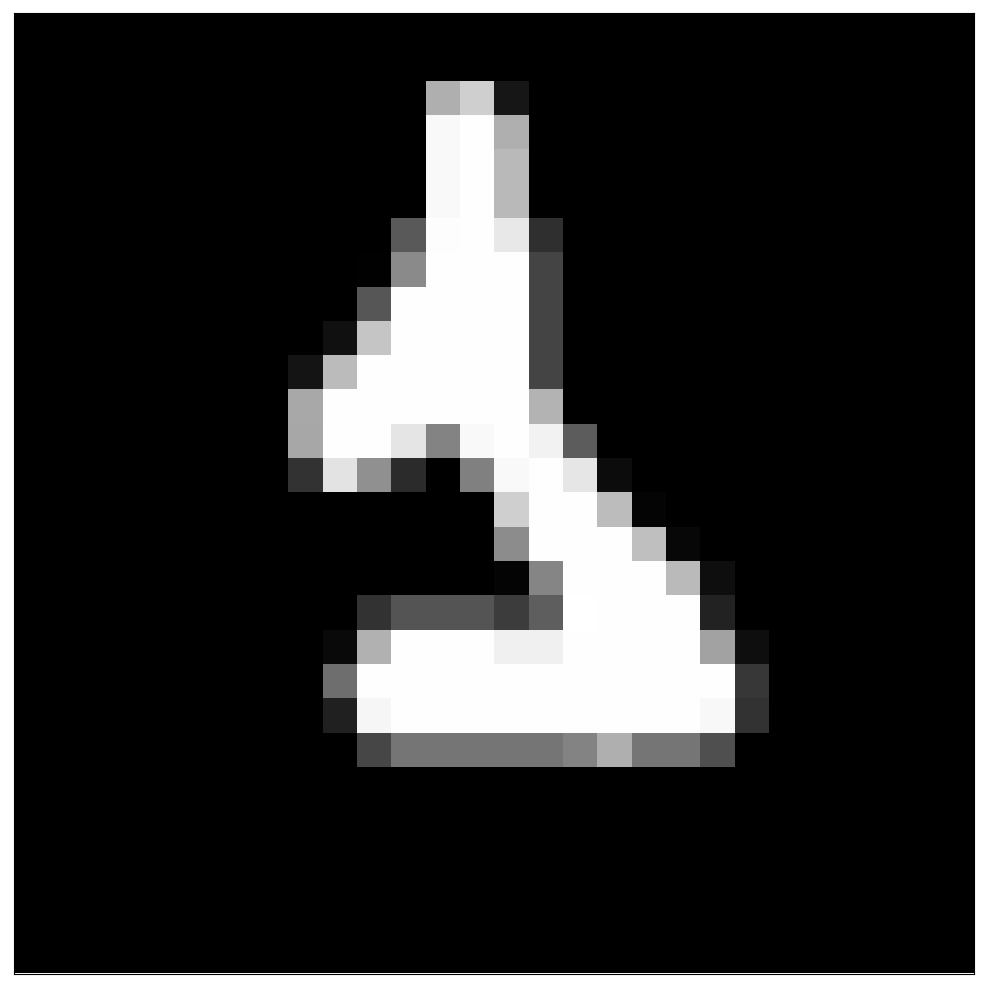}
    &  \includegraphics[width=0.13\columnwidth,valign=m]{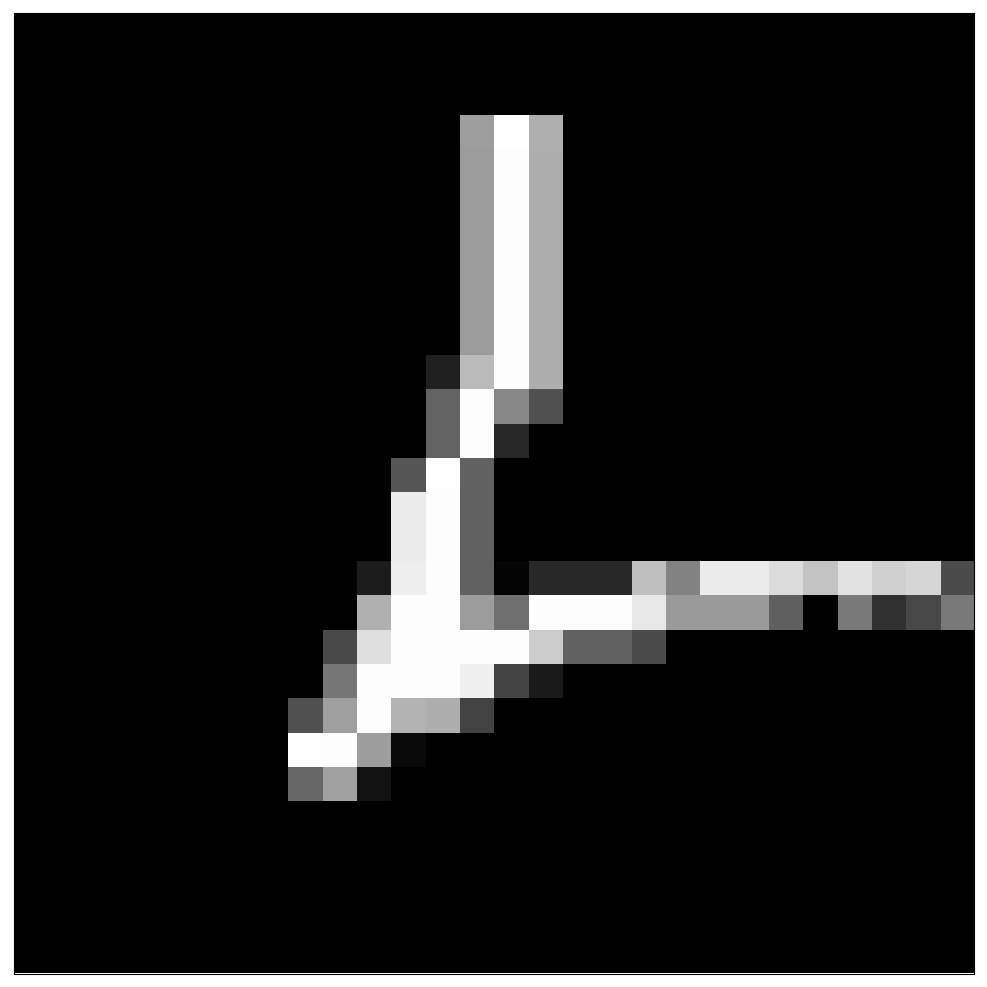}
    & \includegraphics[width=0.13\columnwidth,valign=m]{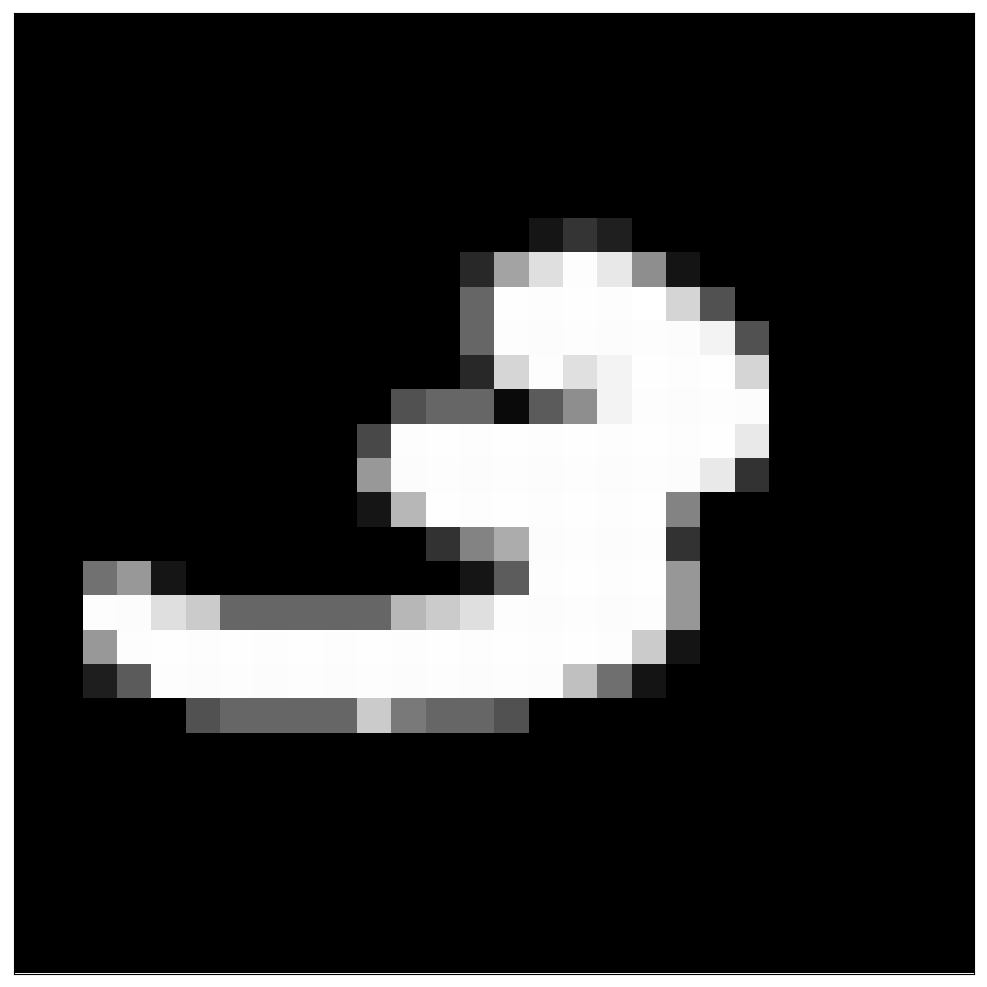}
    & \includegraphics[width=0.13\columnwidth,valign=m]{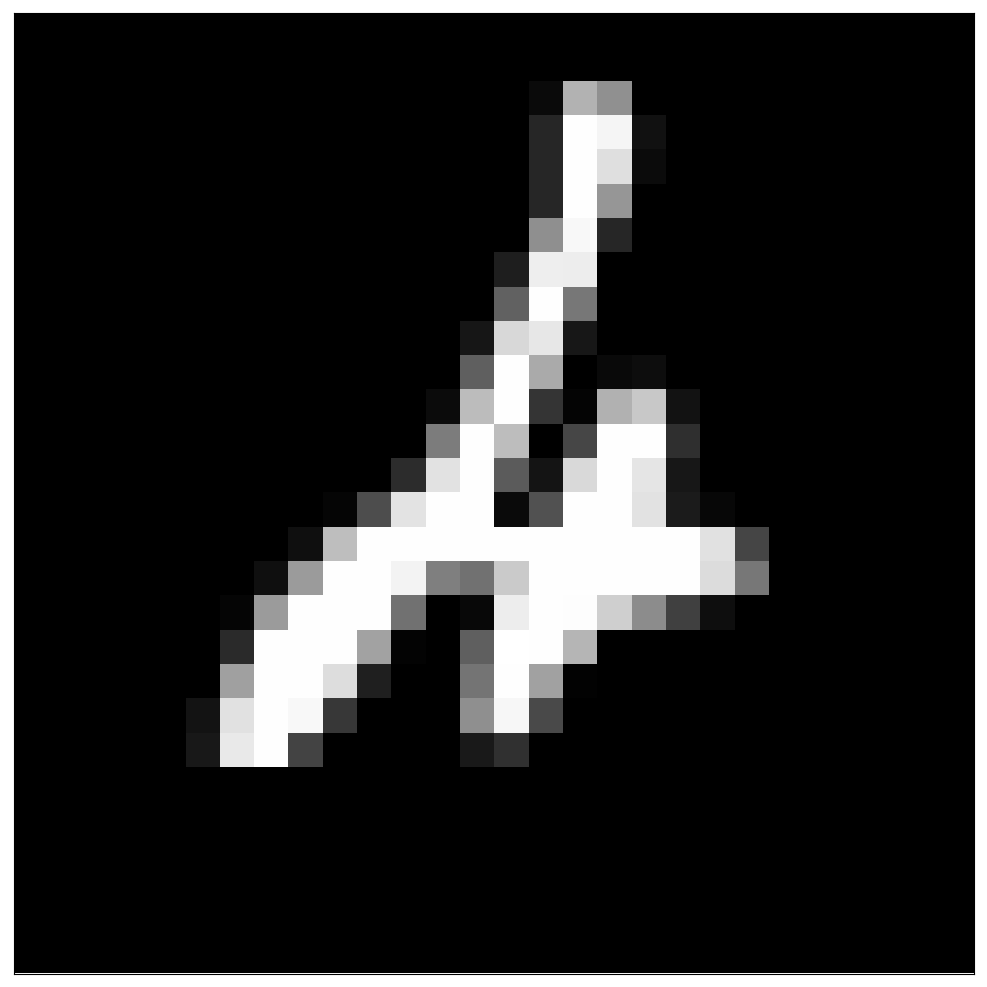}\\ 
    &  $10^{-31}$ & $10^{-40}$ & $10^{-25}$ & $10^{-31}$ & $10^{-38}$\\
    & \includegraphics[width=0.13\columnwidth,valign=m]{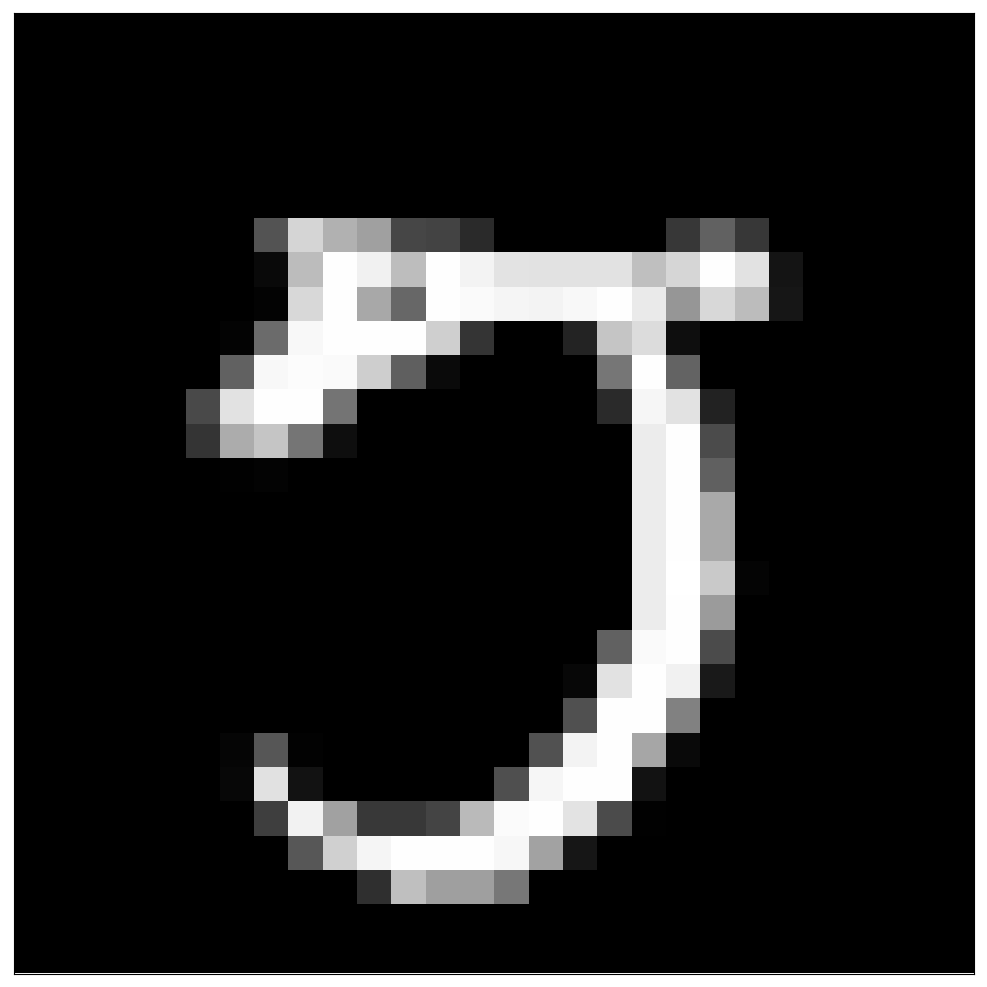}
    & \includegraphics[width=0.13\columnwidth,valign=m]{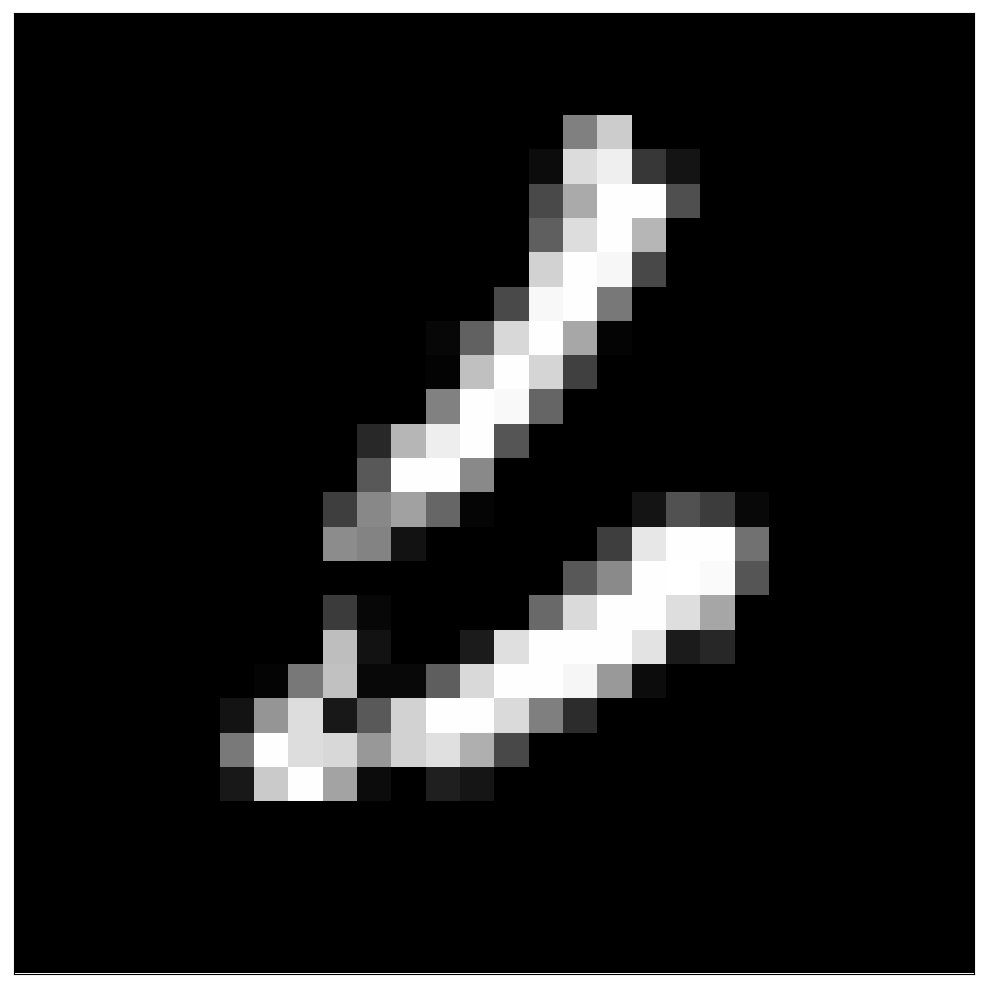}
    & \includegraphics[width=0.13\columnwidth,valign=m]{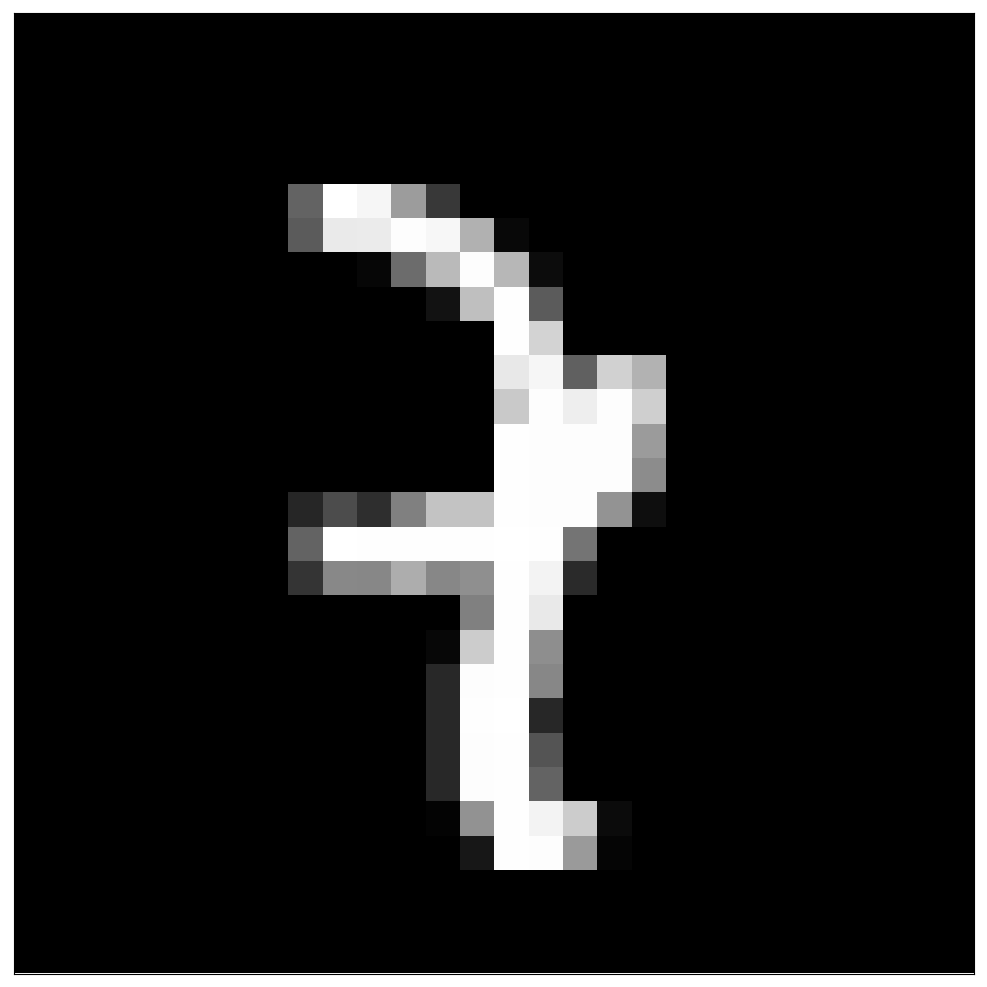}
    & \includegraphics[width=0.13\columnwidth,valign=m]{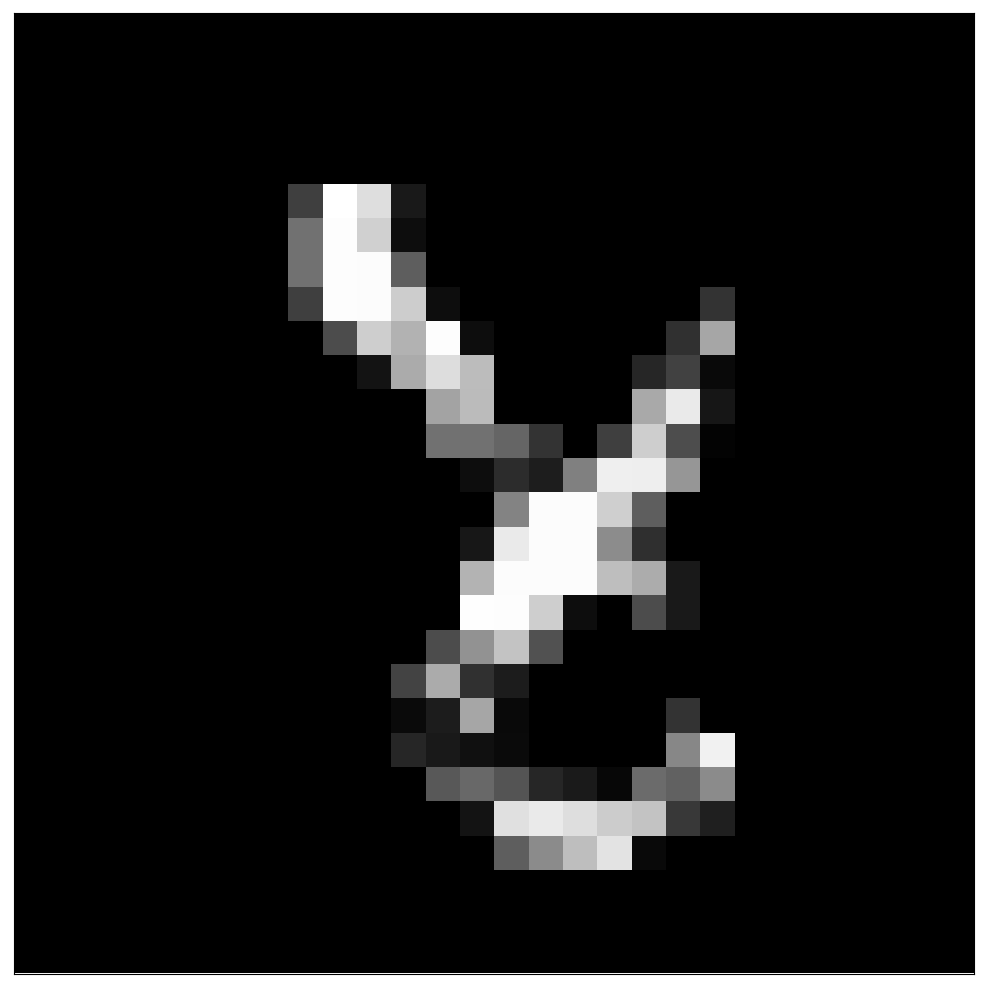}
    & \includegraphics[width=0.13\columnwidth,valign=m]{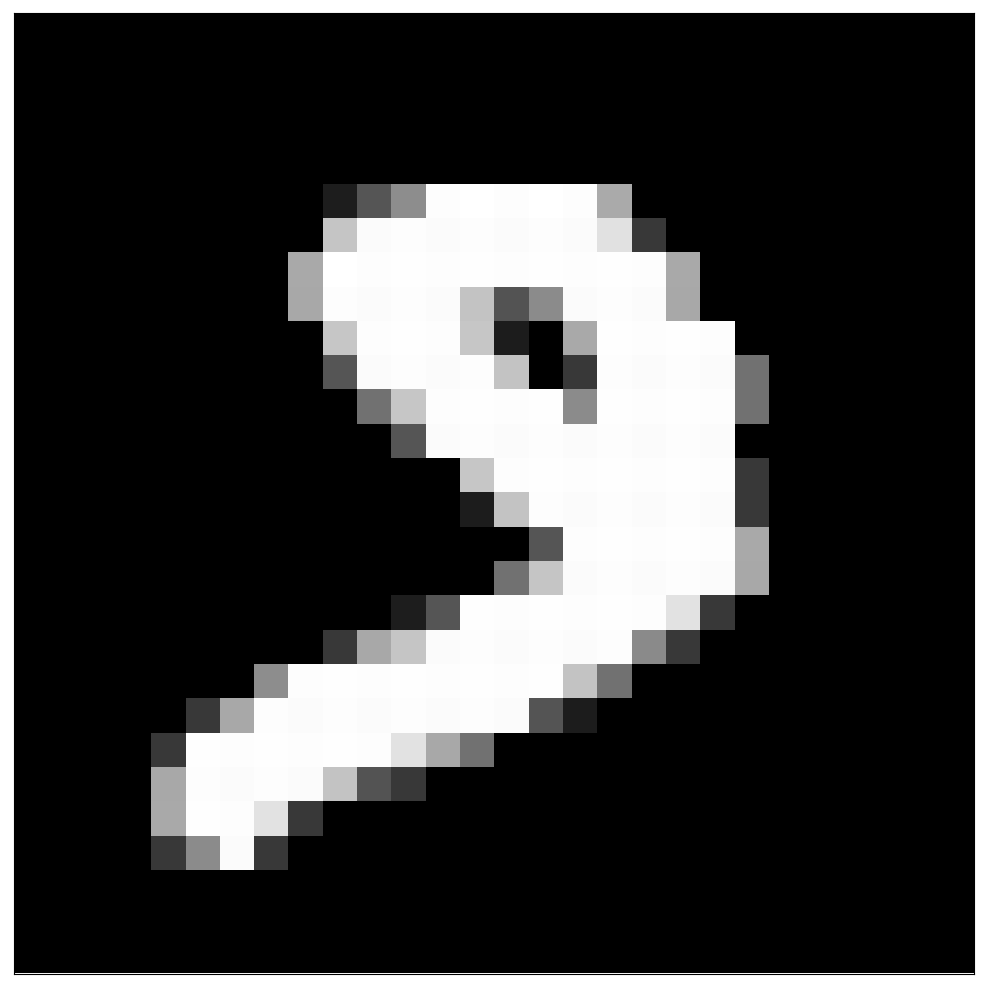}
  \end{tabular}
  \caption{Prototypes (top half) and outliers (bottom half) according to the Gauss confidence measure, with associated Gauss-confidence values.\label{fig:mnist}}
\end{figure}
The monochrome handwritten digits in the MNIST dataset are not vulnerable to one-pixel attacks in any way resembling the magnitude of the problem on Cifar-10, so we cannot simply replicate the experimental results from the previous section on this dataset.  Instead, we illustrate how the predictive accuracy of the two networks (traditional and Gauss network) varies when increasing the level of distortions as provided by the FGSM attack.  We plot the resulting accuracy and confidence of successfully attacked test data points in Figure \ref{fig:noisePlot}. The accuracy reflects the percentage of successfully attacked images in the test dataset.

When the $\epsilon$ parameter increases, all networks at some point display severely deteriorating accuracy. However, the accuracy of the Gauss network does not decrease as rapidly. The gap between the accuracies for Gauss and traditional networks increases with $\epsilon$. Similarly, the higher the distortion, the more confident do the networks become in the confidence of wrong classifications. Both confidence scores approximate the average confidences of the networks in the test dataset ($0.98$ for traditional networks and $0.56$ for Gauss-networks).

We show by means of the MNIST dataset which images come close to prototypes of their class and which images are classified as outliers. We display the images which gain highest, respectively lowest Gauss-confidence scores in their respective class on the top and bottom of Figure~\ref{fig:mnist}. The number on top of each image is the corresponding Gauss confidence. We see that the assessment of prototypes and outliers by the Gauss network largely corresponds to a human perception of well and difficult to interpret ciphers. Thus, the information given by the Gauss confidence actually corresponds with the notion of prototypes and outliers. For comparison, traditional networks generally classify a noise image with confidence higher that $0.9$ to one of the classes while a Gauss network reflects the outlier with confidences close to zero.
\section{Conclusions}

We mathematically prove that the predictions of neural networks based on the softmax activation function are equivalent to $k$-means clustering.  This connection was intuited in existing work, but it has never before been formally derived.  The output of the penultimate layer of the network represents a transformed data space.  That space is partitioned into cones by the softmax function; each cone encompasses a single class.  Conceptually, this is equivalent to putting a number of centroids, to be calculated from the weights of the last layer, at equal distance from the origin in the transformed space; the data points in the dataset are subsequently classified, by clustering them through $k$-means by proximity to the centroids.  We formally derive this connection in Theorem \ref{thm:pen=km} and subsequent Equation (\ref{eq:kmPred}).

The $k$-means/softmax relation can be used to explain why softmax-based neural networks are sensitive to adversarial attacks such as one-pixel attacks.  The effect of misclassification due to small input perturbations had been theoretically analyzed in terms of the Lipschitz continuity of the network function \cite{tsuzuku2018lipschitz}; exact computation of the relevant Lipschitz modulus had been shown to be NP-hard \cite{2018Virmaux}.  We theoretically show that in addition to a small Lipschitz modulus, the robustness of a neural net also depends on the proximity with which confidently classified points are mapped to their corresponding centroid (cf.\@ Theorem~\ref{thm:centroidClose}). Hence, we establish a connection between the robustness of a network and their mapping to a $k$-means-friendly space.


In Section \ref{sec:tailoring}, we introduce an alternative to the softmax function in the last layer of a neural network.  This centroid-based tailored version of the network is theoretically well-founded, since we have explored the relation with $k$-means clustering.  Moreover, the tailoring does not affect the predictive accuracy of the network on the Cifar-10 test set, while substantially reducing the proportion of images whose worst one-pixel attack leads to misclassification: this rate drops by a factor between roughly $2.7$ and $6.8$ (cf.\@ Table \ref{tbl:attackCifar}).  On the MNIST dataset, we have seen how the tailored version of the network achieves highest accuracy when confronted with reasonable levels of noise in handwritten digits (cf.\@ Figure \ref{fig:noisePlot}).  Our new tailored neural network reduces vulnerability to small-noise attacks due to its capability of expressing reasonable doubt; the lack of doubt in traditional networks is illustrated in Figure \ref{fig:statConf}, where we also see that the Gauss network spreads its confidence levels throughout the available space.  This effect is desirable, since the non-boosted Gauss-confidence allows the network to distinguish class prototypes, such as the clearly written digits in the top half of Figure \ref{fig:mnist}, from outliers, such as the abstract worm paintings in the bottom half of Figure \ref{fig:mnist}.

\bibliographystyle{aaai}

\end{document}